\documentclass{article}
\usepackage{nips08submit_e,times}
\usepackage{amsmath}
\usepackage{algorithm}
\usepackage[noend]{algorithmic}
\usepackage{graphicx,epstopdf}

\usepackage{soul,color} 
\usepackage{ifthen}
\newcounter{RevisionNumber} 
\newcommand{\hlc}[2]{\ifthenelse{\value{RevisionNumber}=#1}{\hl{#2}}{#2}}
\newcommand{\ASnote}[2]{\ifthenelse{\value{RevisionNumber}=#1}{\hl{#2}}{}}

\usepackage{amsmath, amssymb, amsthm, verbatim}

\newtheorem{thm}{Theorem}
\newtheorem{claim}{Claim}

\newtheorem{defn}{Definition}

\newcommand{\eps}{\epsilon}
\newcommand{\ce}{e^{1/3}}

\newcommand{\bbR}{\mathbb{R}}

\newcommand{\bbB}{\mathbb{B}}

\newcommand{\CS}{\mathcal{S}}

\newcommand{\CX}{\mathcal{X}}

\newcommand{\CF}{\mathcal{F}}

\newcommand{\goes}{\rightarrow}

\providecommand{\norm}[1]{\lVert#1\rVert}

\newcommand{\be}{\begin{equation}}
\newcommand{\ee}{\end{equation}}
\newcommand{\bez}{\begin{equation*}}
\newcommand{\eez}{\end{equation*}}

\newcommand{\slidetextsize}{\footnotesize}
\providecommand{\bsl}[1]{\begin{slide}{#1}\slidetextsize}
\providecommand{\esl}{\end{slide}}

\newcommand{\OMIT}[1]{}

\newcommand{\regret}{R}
\newcommand{\UCB}{\ensuremath{\textsc{ucb1}}}
\newcommand{\ucb}{\UCB}

\newcommand{\ucbo}{\ensuremath{\textsc{ora}}}
\newcommand{\ucbc}{\ensuremath{\textsc{bwc}}}

\newcommand{\expthree}{\ensuremath{\textsc{exp3.s}}}

\newcommand{\WHP}{\textnormal{w.h.p.}}
\DeclareMathOperator{\FP}{FP}

\newcommand{\refeq}[1]{(\ref{#1})}
\newenvironment{note}[1]
	{\begin{trivlist}
		\item[\hskip \labelsep {\em #1}]}
	{\end{trivlist}}
\theoremstyle{plain}

\newcommand{\A}{\ensuremath{\mathcal{A}}}

\title{Adapting to the Shifting Intent of Search Queries%
\thanks{This is the full version of a paper in \emph{NIPS 2009}. }}

\author{
Umar Syed%
\thanks{This work was done while the author was an intern at Microsoft Research and a student in the Department of Computer Science, Princeton University.} \\
Department of Computer\\
and Information Science\\
University of Pennsylvania\\
Philadelphia, PA 19104\\
\texttt{usyed@cis.upenn.edu} \\
\And
Aleksandrs Slivkins\\
Microsoft Research\\
Mountain View, CA 94043\\
\texttt{slivkins@microsoft.com} \\
\And
Nina Mishra\\
Microsoft Research\\
Mountain View, CA 94043\\
\texttt{ninam@microsoft.com} \\
}

\begin{document}

\maketitle

\vspace{-2mm} 
\begin{abstract}
Search engines today present results that are often oblivious to
abrupt shifts in intent.  For example, the query
`independence day' usually refers to a US holiday, but the intent of this query abruptly changed
during the release of a major film by that name.  While no studies exactly
quantify the magnitude of intent-shifting traffic, studies suggest
that news events, seasonal topics, pop culture, etc account for 50\% of all
search queries.  This paper shows that the signals a search engine
receives can be used to both determine that a shift in
intent has happened, as well as find a result that is now more relevant.
We present a meta-algorithm that marries a classifier with a bandit
algorithm to achieve regret that depends logarithmically on the number
of query impressions, under certain assumptions.  We provide strong
evidence that this regret is close to the best achievable.  Finally,
via a series of experiments, we demonstrate that our algorithm
outperforms prior approaches, particularly as the amount of
intent-shifting traffic increases.
\end{abstract}

\section{Introduction}
\label{sec:intro}

Search engines typically use a ranking function to order results.  The
function scores a document by the extent to which it matches the
query, and documents are ordered according to this score.  Usually, this
function is fixed in the sense that it does not change from
one query to another and also does not change over time.  

Intuitively, a query is ``intent-shifting'' if the most desired search result(s)
change over time.  More concretely, a query's intent has shifted if the click
distribution over search results at some time differs from the click
distribution at a later time.  For the query `tomato' on the heels of
a tomato salmonella outbreak, the probability a user clicks on a news
story describing the outbreak increases while the probability a user
clicks on the Wikipedia entry for tomatoes rapidly decreases.
There are studies that suggest that queries likely to be intent-shifting --- such as pop culture, news
events, trends, and seasonal topics queries --- constitute roughly half of
the search queries that a search engine receives~\cite{pew}. 


The goal of this paper is to devise an algorithm that quickly adapts
search results to shifts in user intent. Ideally, for every
query and every point in time, we would like to display the search
result that users are most likely to click. Since traditional ranking
features like PageRank~\cite{pagerank} change slowly over time, and may be
misleading if user intent has shifted very recently, we want to use
just the observed click behavior of users to decide which search
results to display.

There are many signals a search engine can use to detect when the intent of a query shifts. Query features such as
as volume, abandonment rate, reformulation rate, occurrence in news articles, and the age of matching documents can all be used to build a classifier
which, given a query, determines whether the intent has shifted.  We
refer to these features as the {\em context}, and an occassion when a shift in intent occurs as an {\em event}.

One major challenge in building an event classifier is
obtaining training data.  For most query and date combinations
(e.g. `tomato, 06/09/2008'), it will be difficult even for a
human labeler to recall in hindsight whether an event related to the
query occurred on that date.  In this paper, we propose a novel
solution that learns from unlabeled contexts and user click activity.

{\bf Contributions.} We describe a new algorithm that leverages the information contained in contexts, provided that such information is sufficiently rich. Specifically, we assume that there exists a deterministic oracle (unknown to the algorithm) which inputs the context and outputs a correct binary prediction of whether an event has occurred in the current round. To simulate such an oracle, we use a classification algorithm. However, we do \emph{not} assume that we have a priori labeled samples to train such a classifier. Instead, we generate the labels ourselves.

Our algorithm is in fact a meta-algorithm that combines a
bandit algorithm designed for the event-free setting with an online
classification algorithm. The classifier uses the contexts to predict
when events occur, and the bandit algorithm ``starts over'' on
positive predictions. The bandit algorithm provides feedback to the
classifier by checking, soon after each of the classifier's positive
predictions, whether the optimal search result actually changed. In such a setup, one needs to overcome several technical hurdles, e.g. ensure that the feedback is not ``contaminated" by events in the past and in the near future. We design the whole triad --- the bandit algorithm, the classifier, and the meta-algorithm --- so as to obtain strong provable guarantees. Our bandit subroutine --- a novel version of algorithm $\ucb$ from~\cite{bandits-ucb1} which additionally provides high-confidence estimates on the suboptimality of arms --- may be of independent interest.



For suitable choices of the bandit and classifier subroutines, the
regret incurred by our meta-algorithm is (under certain mild assumptions)
at most $O(k + d_\CF) (\frac{n}{\Delta} \log T)$, where $k$ is the number of events,
$d_\CF$ is a certain measure of the complexity of the
concept class $\CF$ used by the classifier, $n$ is the number of relevant search results,\footnote{In practice, the arms can be restricted to, say, the top ten results that match the query.} $\Delta$ is the ``minimum suboptimality'' of any search result (defined formally in Section \ref{sec:prelim}), and $T$ is the total
number of impressions. This regret bound has a very weak dependence on $T$, which is
highly desirable for search engines that receive much traffic.

The context turns out to be crucial for achieving logarithmic
dependence on $T$.  Indeed, we show that any bandit algorithm that
ignores context suffers regret $\Omega(\sqrt{T})$, even when there is only one event.  Unlike many lower
bounds for bandit problems, our lower bound holds even when $\Delta$
is a constant independent of $T$. We also show that assuming a logarithmic dependence on $T$, the dependence on $k$ and $d_\CF$ is essentially optimal.

\OMIT{
We also prove that our algorithm is optimal, in a certain sense. In
particular, we prove that for any algorithm $A$, there are problems for which
either $\regret(A) \ge ((k + d_\CF) \log T)$, or $\regret(A) \ge \Omega(T^{1/3})$. In other words,
any algorithm that suffers logarithmic regret in $T$ must also suffer
regret that depends on the number of events and the complexity of the
concept class.}


For empirical evaluation, we ideally need access to the traffic of a
real search engine so that search results can be adapted based on
real-time click activity.  Since we did not have access to live
traffic, we instead conduct a series of synthetic experiments.  The
experiments show that if there are no events then the
well-studied $\ucb$ algorithm~\cite{bandits-ucb1} performs the best.  However, when many different queries
experience events, the performance
of our algorithm significantly outperforms prior methods.

\section{Related Work}
While there has been a substantial amount of work on ranking
algorithms~\cite{rankboost,ranknet,ranksvm,orderthings,listwise},
all of these results assume that there is a fixed ranking
function to learn, not one that shifts over time.  Online bandit
algorithms (see~\cite{CesaBL-book} for background) have been considered in the context of ranking.  For
instance, Radlinski et al~\cite{RBA-icml08} showed how to compose
several instantiations of a bandit algorithm to produce a ranked list
of search results. Pandey et al \cite{yahoo-bandits-icml07} showed that
bandit algorithms can be effective in serving advertisements to search
engine users.  These approaches also assume a stationary inference problem.

Although no existing bandit algorithms are specifically designed for our problem setting,
there are two well-known algorithms that we compare against in this paper.  The
$\ucb$ algorithm~\cite{bandits-ucb1} assumes fixed click probabilities
and has regret at most $O(\frac{n}{\Delta} \log T)$.  The $\expthree$
algorithm~\cite{bandits-exp3} assumes that click probabilities can
change on every round and has regret at most $O(k\sqrt{nT\log(nT)})$
for arbitrary $p_t$'s.  Note that the dependence of $\expthree$ on $T$ is
substantially stronger.

The ``contextual bandits'' problem setting~\cite{Wang-sideMAB05,yahoo-bandits07,Hazan-colt07,contextual,kakade} is similar to
ours.  A key difference is that the context received in each round
is assumed to contain information about the \emph{identity} of an
optimal result $i^*_t$, a considerably stronger assumption than we
make.  Our context includes only side information such as volume of the
query, but we never actually receive information about the identity of
the optimal result.

A different approach is to build a statistical model of user click
behavior.  This approach has been applied to the problem of serving
news articles on the web. Diaz~\cite{diaz} used a regularized logistic
model to determine when to surface news results for a query.
Agarwal et al~\cite{agarwalchen} used several models, including a
dynamic linear growth curve model.

There has also been work on detecting bursts in data streams.
For example, Kleinberg~\cite{bursty} describes a state-based model
for inferring stages of burstiness.
The goal of our work is not to detect bursts, but rather to
predict shifts in intent.

In a recent concurrent and independent work, Yu et al~\cite{Yu-icml09} studied bandit problems with
``piecewise-stationary'' distributions, a notion that closely resembles
our definition of events. However, they make different assumptions
than we do about the information a bandit algorithm can
observe. Expressed in the language of our problem setting, they assume
that from time-to-time a bandit algorithm receives information about
how users \emph{would have} responded to search results that are never
actually displayed. For our setting, this assumption is clearly inappropriate.

\section{Problem Formulation and Preliminaries}
\label{sec:prelim}

We view the problem of deciding which search results to display in
response to user click behavior as a \emph{bandit
problem}, a well-known type of sequential decision problem. For a
given query $q$, the task is to determine, at each round $t \in \{1,
\ldots, T\}$ that $q$ is issued by a user to our search engine, a
single result $i_t \in \{1, \ldots, n\}$ to display.\footnote{For
simplicity, we focus on the task of returning a single result, and not
a list of results.  Techniques from~\cite{RBA-icml08} may be adopted
to find a good list of results.} This
result is clicked by the user with probability $p_t(i_t)$. A bandit
algorithm $\mathcal{A}$ chooses $i_t$ using only observed information from
previous rounds, i.e., all previously displayed results and
received clicks. The performance of an algorithm $\mathcal{A}$ is measured by
its \emph{regret}:
$\regret_\mathcal{A}(T)
	\triangleq E\left[\sum_{t=1}^T p_t(i^*_t) - p_t(i_t)\right]$,
where an \emph{optimal}
result $i^*_t \in \arg \max_i p_t(i)$ is one with maximum click
probability, and the expectation is taken over the randomness in the
clicks and the internal randomization of the algorithm. Note our
unusually strong definition of regret: we are competing against the
best result on \emph{every} round.

We call an \emph{event} any round $t$ where $p_{t-1} \neq p_t$. It
is reasonable to assume that the number of events $k \ll T$, since we
believe that abrupt shifts in user intent are relatively rare. Most
existing bandit algorithms make no attempt to predict when events will
occur, and consequently suffer regret $\Omega(\sqrt{T})$.
On the other hand, a typical search engine receives many
signals that can be used to predict events, such as bursts in query
reformulation, average age of retrieved document, etc.

We assume that our bandit algorithm receives a \emph{context} $x_t \in \CX$ at each round $t$, and that there exists a function $f \in \CF$, in some known \emph{concept class} $\CF$, such that $f(x_t) = +1$ if an event occurs at round $t$, and $f(x_t) = -1$ otherwise.\footnote{In some of our analysis, we require contexts be restricted to a strict (concept-specific) subset of $\CX$; the value of $f$ outside this subset will technically be {\tt null}. See Section~\ref{sec:safe} for more details.} In other words, $f$ is an \emph{event oracle}. The tractability of $\CF$ will be characterized by a number $d_\CF$ called the \emph{diameter} of $\CF$, detailed in Section~\ref{sec:safe}. At each round $t$, an \emph{eventful bandit algorithm} must choose a result $i_t$ using only observed information from previous rounds, i.e., all previously displayed results and received clicks, plus all contexts up to round $t$.



In order to develop an efficient eventful bandit algorithm, we make an additional key assumption: At least one optimal result before an event is \emph{significantly} suboptimal after the event. More precisely, we assume there exists a \emph{minimum shift} $\eps_S > 0$ such that, whenever an event occurs at round $t$, we have $p_t(i^*_{t-1}) < p_t(i^*_t) - \eps_S$ for at least one previously optimal search result $i^*_{t-1}$. For our problem setting, this assumption is relatively mild: the events we are interested in tend to have a rather dramatic effect on the optimal search results. Moreover, our bounds are parameterized by
	$\Delta = \min_t \min_{i \neq i^*_t} p_t(i^*_t) - p_t(i)$,
the \emph{minimum suboptimality} of any suboptimal result.

We summarize the notation in Table~\ref{tab:notation}.

\begin{table}
\caption{Notation}
\label{tab:notation}
\begin{tabular}{cc}

\begin{minipage}[c]{0.48\columnwidth}
\begin{center}
\begin{tabular}{l|l}
 	$\CX$ & universe of contexts \\
 	$x_t \in \CX$ & context in round $t$ \\
  	$\CF$ & concept class \\
  	$f\in \CF$ & concept \\
  	$d_\CF$ & diameter of $\CF$ \\
\end{tabular}
\end{center}
\end{minipage}

&

\begin{minipage}[c]{0.48\columnwidth}
\begin{center}
\begin{tabular}{l|l}
  	$k$ & number of events \\
 	$n$ & number of arms \\
  	$T$ & time horizon \\
  	$\eps_S$ & min shift of an event\\
  	$\Delta$ & min suboptimality of an arm
\end{tabular}
\end{center}
\end{minipage}

\end{tabular}
\end{table}

Let $S$ be the set of all contexts which correspond to an event. When the classifier receives a context $x$ and predicts a ``positive", this prediction is called a \emph{true positive} if $x\in S$, and a \emph{false positive} otherwise. Likewise, when the classifier predicts a ``negative", the prediction is called a \emph{true negative} if $x\not\in S$, and a \emph{false negative} otherwise. The sample $(x,l)$ is \emph{correctly labeled} if $l = (x \in S)$.


\section{Bandit with Classifier}
\label{sec:BWC}

\newcommand{\AlexDiam}{FP-complexity}
\newcommand{\MinShift}{minimum shift}

Our algorithm is called $\ucbc$, or ``Bandit with Classifier''. Ideally, we would like to use a bandit algorithm for the event-free setting, such as $\ucb$, and restart it every time there is an event. Since we do not have an oracle to tell whether an event has happened, we use a classifier which looks at the current context and makes a binary prediction. As we mentioned in the introduction, we assume that a priori there are no labeled samples to train such a classifier, so we need to generate the labels ourselves. The high-level idea is to restart the bandit algorithm every time the classifier predicts an event, and use subsequent rounds to generate feedback (labeled samples) to train the classifier. Thus, we have a \emph{feedback loop} between the bandit algorithm and a classifier, in which the latter provides predictions and the former verifies whether they are correct, see Figure~\ref{fig:loop}.

\begin{figure}[!h]
\begin{center}
\includegraphics[scale=0.3]{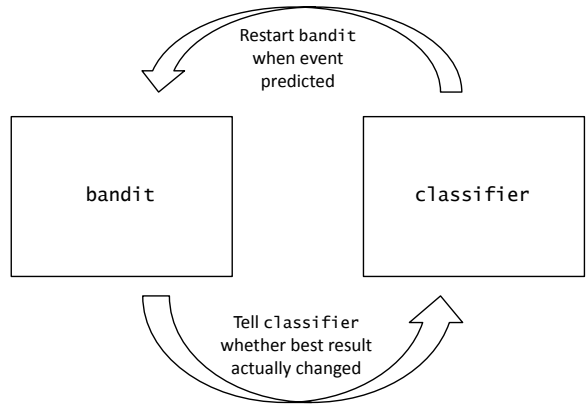}
\end{center}
\caption{A high-level picture of the operation of the $\ucbc$ algorithm, depicting the feedback loop between the two main subroutines, {\tt bandit} and {\tt classifier}.}
\label{fig:loop}
\end{figure}

So what prevents us from simply combining an off-the-shelf bandit algorithm with an off-the-shelf classifier? The central challenge is how to define the feedback. Let us outline several hurdles that we need to overcome here. A single false negative prediction will cause $\ucbc$ to miss an event, which may result in a very high regret (since it may take the bandit algorithm a very long time to adjust). Incorrectly labeled samples may contaminate the classifier, perhaps even permanently. To generate a label for a given sample, one needs to compare the state right before the current round with the state right after, in a conclusive way. Both states are not known to the algorithm a priori, and can only be learned probabilistically via exploration. A particular challenge is to ensure that such exploration is not contaminated by events in the past rounds, as well as by events that happen soon after the current round. Moreover, this exploration is generally too expensive to perform upon \emph{negative} predictions --- indeed, the whole point of $\ucbc$ is that in the absence of an event the bandit algorithm converges to the best arm and (essentially) keeps playing it --- so the classifier receives labels only upon the positive predictions.

\subsection{The meta-algorithm}

We will present our algorithm in a modular way, as a \emph{meta-algorithm} which uses the following two components: {\tt classifier} and {\tt bandit}. In each round, {\tt classifier} inputs a context $x_t$ and outputs a ``positive" or ``negative" prediction of whether an event has happened in this round. Also, it may input labeled samples of the form $(x,l)$, where $x$ is a context and $l$ is a boolean label, which it uses for training. Algorithm {\tt bandit} is a bandit algorithm that is tuned for the event-free runs.

As described above, we further require {\tt bandit} to provide feedback to the classifier about whether the best result has actually changed. The standard bandit framework does not immediately provide us with estimates from which such feedback can be obtained. Therefore we require {\tt bandit} to provide the following additional functionality: after each round $t$ of execution, it outputs a pair $(G^+, G^-)$ of subsets of arms;\footnote{Following established convention, we call the options available to a bandit algorithm ``arms''. In our setting, each arm corresponds to a search result.} we call this pair the \emph{$t$-th round guess}.\footnote{Both {\tt classifier} and {\tt bandit} make predictions (about events and arms, respectively). For clarity, we will use the term ``guess'' exclusively to refer to predictions made by {\tt bandit}, and reserve the term ``prediction'' for {\tt classifier}.} The meaning of $G^+$ and $G^-$ is that they are algorithm's estimates for, respectively, the sets of all optimal and (at least) $\eps_S$-suboptimal arms. We use $(G^+, G^-)$ to predict whether an event has happened between two  runs of {\tt bandit}. The idea is that any such event causes some arm from $G^+$ of the first run to migrate to $G^-$ of the second run. Accordingly, we generate a negative label if $G^+_i \cap G^-_j = \emptyset$, where $i$ and $j$ refers to the first and the second run, respectively (see Line 10 of Algorithm~\ref{alg}).

\OMIT{Formally, we require a weaker property: essentially, for a large enough $t$ each optimal arm lies in $G^+$ but not in $G^-$, and any arm that is at least $\eps_S$-suboptimal lies in $G^-$ but not in $G^+$.}

We formalize our assumptions on {\tt classifier} and {\tt bandit} as follows:

\begin{defn}
{\tt classifier} is {\bf safe} for a given concept class if, given only correctly labeled samples, it never outputs a false negative. {\tt bandit} is called {\bf $(L,\eps)$-testable}, for some $L\in\mathbb{N}$ and $\eps\in(0,1)$, if the following holds. Consider an event-free run of {\tt bandit}, and let $(G^+, G^-)$ be its $L$-th round guess. Then with probability at least $1-T^{-2}$, each optimal arm lies in $G^+$ but not in $G^-$, and any arm that is at least $\eps$-suboptimal lies in $G^-$ but not in $G^+$.~\footnote{Recall that $T$ here is the overall time horizon, as defined in Section~\ref{sec:prelim}.}
\end{defn}

We will discuss efficient implementations of a safe {\tt classifier} and a $(L, \eps_S)$-testable {\tt bandit} in Sections~\ref{sec:safe} and Section~\ref{sec:UCB}, respectively. For {\tt bandit}, we build on a standard algorithm \UCB~\cite{bandits-ucb1}; as it turns out, making it $(L, \eps_S)$-testable requires a significantly extended analysis.

For correctness, we require {\tt bandit} to be $(L,\eps_S)$-testable, where $\eps_S$ is the \MinShift{}. The performance of {\tt bandit} is quantified via its {\bf event-free regret}, i.e. regret on the event-free runs. Likewise, for correctness we need {\tt classifier} to be safe, and we quantify its performance via the following property, termed {\bf \AlexDiam{}}, which refers to the maximum number of false positives.

\begin{defn}\label{def:FP-complexity}
Given a concept class $\CF$, the \AlexDiam{} of {\tt classifier} is the maximum possible number of false positives it can make in an online prediction game where in each round, an adversary selects a sample, {\tt classifier} makes a prediction, and then (in some rounds) receives a correct label. Specifically, {\tt classifier} receives a correct label if and only if the prediction is a false positive. The maximum is taken over all event oracles $f\in\CF$ and all possible sequences of samples.
\end{defn}

Now we are ready to present our meta-algorithm, called $\ucbc$. It runs in phases of two alternating types: odd phases are called ``testing'' phases, and even phases are called ``adapting'' phases. The first round of phase $j$ is denoted $t_j$. In each phase we run a fresh instance of {\tt bandit}. Each testing phase lasts for $L$ rounds, where $L$ is a parameter.  Each adapting phase $j$ ends as soon as {\tt classifier} predicts ``positive"; the round $t$ when this happens is round $t_{j+1}$.  Phase $j$ is called \emph{full} if it lasts at least $L$ rounds. For a full phase $j$, let $(G^+_j, G^-_j)$ be the $L$-th round guess in this phase. After each testing phase $j$, we generate a boolean prediction $l$ of whether there was an event in the first round thereof. Specifically, letting $i$ be the most recent full phase before phase $j$, we set $l_{t_j} = \texttt{false}$ if and only if $G^+_i \cap G^-_j = \emptyset$. If $l_{t_j}$ is {\tt false}, the labeled sample $(x_{t_j}, l_{t_j})$ is fed back to the classifier. Note that {\tt classifier} never receives {\tt true}-labeled samples. Pseudocode for $\ucbc$ is given in Algorithm \ref{alg}.

Disregarding the interleaved testing phases for the moment, $\ucbc$ restarts {\tt bandit} whenever {\tt classifier} predicts ``positive'', optimistically assuming that the prediction is correct. By our assumption that events cause some optimal arm to become significantly suboptimal (see Section \ref{sec:prelim}), a correct prediction should result in $G^+_i \cap G^-_j \neq \emptyset$, where $i$ is a phase before the putative event, and $j$ is a phase after it. We use this condition in Line 10 of the pseudocode to generate the label. However, to ensure that the estimates $G_i$ and $G_j$ are reliable, we require that phases $i$ and $j$ are full. And to ensure that the full phases closest to a putative event are not too far from it, we interleave a full testing phase every other phase.

\begin{algorithm}[t]
\caption{Meta-algorithm $\ucbc$ (``Bandit with Classifier")}
\label{alg}
\begin{algorithmic}[1]
\STATE {\bf Given:} Parameter $L$, a $(L, \eps_S)$-testable {\tt bandit}, and a safe {\tt classifier}.
\FOR{phase $j = 1,2,\, \ldots\; $}
\STATE Initialize {\tt bandit}. Let $t_j$ be current round.
\IF{$j$ is odd}
\STATE \COMMENT{testing phase}
\FOR{round $t = t_j\; \ldots\; t_j + L$}
\STATE Select arm $i_t$ according to {\tt bandit}.
\STATE Observe click w.p. $p_t(i_t)$ and update {\tt bandit}.
\ENDFOR
\STATE Let $i$ be the most recent full phase before phase $j$.
\STATE If $G^+_i \cap G^-_j = \emptyset$
			~~~~~\COMMENT{label is \texttt{false}}
\STATE ~~~~~Let $l_{t_j} = $ \texttt{false} and pass training example $(x_{t_j}, l_{t_j})$ to {\tt classifier}.
\ELSE
\STATE \COMMENT{adapting phase}
\FOR{round $t = t_j,\, t_j+1,\, \ldots\; $}
\STATE Select arm $i_t$ according to {\tt bandit}.
\STATE Observe click w.p. $p_t(i_t)$ and update {\tt bandit}; pass context $x_t$ to {\tt classifier}.
\IF{{\tt classifier} predicts ``positive''}
\STATE Terminate inner for loop.
\ENDIF
\ENDFOR
\ENDIF
\ENDFOR
\end{algorithmic}
\end{algorithm}

\subsection{Provable guarantees}

We present provable guarantees for $\ucbc$ in a modular way, in terms of \AlexDiam{}, event-free regret, and the number of events. This is the main technical result in the paper.

\begin{thm}\label{thm:meta-algorithm}
Consider an instance of the eventful bandit problem with number of rounds $T$, $n$ arms, $k$ events and \MinShift{} $\eps_S$; assume that any two events are at least $2L$ rounds apart. Consider algorithm $\ucbc$ with parameter $L$ and components {\tt classifier} and {\tt bandit} \hlc{1}{that are, respectively, safe and $(L,\eps_S)$-testable.}
\hlc{1}{Suppose the event-free regret of {\tt bandit} is bounded from above by a concave function $R_0(\cdot)$. Then the regret of $\ucbc$ is}

\begin{align}\label{eq:meta-algorithm}
	R_{\ucbc}(T) \leq (2k+d_{\FP})\,R_0\left(\tfrac{T}{2k+d_{\FP}}\right) + (k+d_{\FP})\, R_0(L) + kL,
\end{align}
where \hlc{2}{$d_{\FP}$} is the \AlexDiam{} of {\tt classifier}.
\end{thm}

\begin{note}{\hlc{1}{Remark.}}
We define a safe classifier whose \AlexDiam{} is bounded in terms of some properties of the underlying concept class (see Section~\ref{sec:safe}). Our instantiation of {\tt bandit} is $(L,\eps_S)$-testable for
    $L = \Theta(\tfrac{n}{\eps_S^2} \log T)$,
with concave event-free regret matching that of $\ucb$ (see Section~\ref{sec:UCB}).
\end{note}

\begin{note}{\hlc{1}{Remark.}}
The right-hand side of~\refeq{eq:meta-algorithm} can be parsed as follows. The three summands in~\refeq{eq:meta-algorithm} correspond to contributions of, respectively, adapting phases, event-free testing phases, and testing phases during which an event has occurred. For the first summand, we show that $\ucbc$ incurs regret $R_0(t)$ for each adapting phase of length $t$, bound the number of adapting phases by \hlc{2}{$2k+d_{\FP}$}, and then bound the total contribution of all such phases using concavity. For the second summand, we bound the number of clean testing phases by \hlc{2}{$k+d_{\FP}$}, and note that each such phase contributes at most $R_0(L)$ to regret. For the third summand, each ``eventful" testing phase contributes at most $L$ to regret, and we show that there can be at most $k$ such phases.\footnote{In fact, the $k$ in the $+kL$ term in~\refeq{eq:meta-algorithm} can be replaced by the (potentially much smaller) number of testing phases that contain both a false positive in round $1$ of the phase and an actual event later in the phase.}
\end{note}

\begin{note}{\hlc{1}{Remark.}}
Assuming that any two events are at least $2L$ rounds apart ensures that of any two consecutive phases, one much be event-free. This, in turn, let us invoke the $(L,\eps_S)$-testability of {\tt bandit}.
\end{note}

\begin{proof}[Overview of the proof]
\hlc{1}{The essential difficulty the analysis of $\ucbc$ is that an event might happen while the algorithm is testing for another (suspected) event. The corresponding technical difficulty} is that the correct operation of the components of $\ucbc$ --- {\tt classifier} and {\tt bandit} --- is interdependent, \hlc{1}{so one needs to be careful to avoid a circular argument.} In particular, one challenge is to handle events that occur during the first $L$ rounds of a phase; these events may potentially ``contaminate" the $L$-th round guesses and cause incorrect feedback to {\tt classifier}.

First, \hlc{1}{we argue away the probabilistic nature of the problem}. We focus on a given testing phase $j$. For each of the two preceding phases $i\in \{ j-1, j-2\}$, consider the number \hlc{1}{$N$} of events between the first round of phase $i$ and the first round of phase $j$. We would like to \hlc{1}{establish the following \emph{separation property}: that we can separate (tell apart) the cases of $N=0$ and $N=1$} using the testing condition in Line 10 of the pseudocode. Capitalizing on $(L,\eps_S)$-testability, we define a technical condition which \hlc{1}{implies the separation property} with very high probability. Regret incurred if \hlc{1}{the implication ``technical condition $\Rightarrow$ separation property" \footnote{In the full proof, this implication is called \emph{well-detectability}.} fails to hold} is negligible. Thus we can assume that \hlc{1}{this implication holds always}, and argue deterministically from now on.

\hlc{1}{It is worth noting that we consider \emph{two} preceding phases $i\in \{ j-1, j-2\}$ because either can be used in in Line 10 of the pseudocode (depending on whether phase $j-1$ is full). A crucial point here is that one of these two phases must be event-free.}

Second, we argue that {\tt classifier} receives only correctly labeled samples. We do it in two steps. Using the well-detectable property, we show  that if {\tt classifier} receives an incorrectly labeled sample after some testing phase $j$, then an event must have occurred during the (adapting) phase $j-1$. Then using the safety property of {\tt classifier}, we prove that \hlc{1}{each adapting phase is event-free.}

Third, we bound from above the number of testing and adapting phases, using the maximal number of events and the \AlexDiam{} of the classifier. To this end, we establish that if during a testing phase $j$ there are no events, and furthermore there are no events during the two preceding phases, then in the end of $j$ $\ucbc$ generates a correct label $l = \texttt{false}$. \hlc{1}{Then the regret bound}~\refeq{eq:meta-algorithm} follows easily from the event-free regret of {\tt bandit}.
\end{proof}

Now let us present the full proof which fills the gaps in the above overview.

\begin{proof}[Full Proof]
Let $t_j$ be the first round of phase $j$. Recall that phase $j$ is called \emph{full} if it lasts at least $L$ rounds. For a full phase $j$, let us say that the phase is \emph{\hlc{1}{event-free}} if no events happened during interval $(t_j,\, t_j+L]$, and let $(G^+_j, G^-_j)$ be the $L$-th round guess in this phase. For two full phases $i<j$, let us write \hlc{1}{$i\oplus j$} if and only if $G^+_i \cap G^-_j = \emptyset$. Recall that $i\oplus j$ (as a boolean property) is our algorithm's estimate of whether there was no event in round $t_j$.

A testing phase $j$ is called \emph{well-detectable} if for each phase $i\in \{j-2,j-1\}$ the following property holds: if phases $i$ and $j$ are full and event-free, then: (i) if there are no events in the interval $(t_i, \,t_j]$ then $i\oplus j$, (ii) if in the interval $(t_i, \,t_j]$ there is exactly one event, then $\neg(i \oplus j)$. Since {\tt bandit} is $(L, \eps_S)$-testable, each testing phase $j$ is well-detectable with probability at least $1-2T^{-2}$. Thus, with probability at least $1-\Omega(T^{-1})$ each testing phase is well-detectable. Thus, regret incurred in the case that a phase fails to be well-detectable is negligible. So in the rest of the proof, we will assume that each testing phase is well-detectable.

We claim that if {\tt classifier} receives an incorrectly labeled sample after some testing phase $j$, then an event must have occurred during the (adapting) phase $j-1$. Indeed, by the algorithm specification this sample is $(x_{t_j}, \texttt{false})$, where $t_j$ is the first round of phase $j$. Thus, an event has happened in round $t_j$, and yet we have $i\oplus j$, where $i$ is the most recent full phase before phase $j$. Since each testing phase is well-detectable, it follows that at least one more event happened between the beginning of phase $i$ and the end of phase $j$. Since any two events are at least $2L$ rounds apart, phase $i$ started at some round $t_i<t_j-2L$, and an event has happened in the interval $[t_i,\, t_j-2L)$. To prove the claim, it suffices to show that $i=j-1$. Now, if phase $j-1$ lasted less than $L$ steps, then $i = j-2$ is a testing phase, and so $t_i\geq t_j-2L$, contradiction. Thus phase $j-1$ lasted at least $L$ steps, and so $i = j-1$, claim proved.

\hlc{1}{We claim that all adapting phases are event-free.} For the sake of contradiction, suppose an event occurs during an adapting phase, and let $t$ be the first round at which this happens. We know that {\tt classifier} output a (false) negative in this round, since otherwise a new testing phase would have started at round $t$. Since {\tt classifier} is safe, at some round before $t$ it must have received an incorrectly labeled sample. By the algorithm specification, this must have happened after some testing phase $j$ which ended before round $t$. But then (by the previous claim) an event must have occurred during the (adapting) phase $j-1$, which contradicts the choice of $t$. Claim proved.

From the previous two claims, it follows that {\tt classifier} receives only correctly labeled samples.

We claim that if there are no events during some testing phases $j-2$ and $j$, then at the end of phase $j$ we generate a label $l = \texttt{false}$. Indeed, suppose not. Then $\neg(i\oplus j)$, where $i$ is the most recent full phase before phase $j$. Either $i = j-2$ or $i = j-1$; in either case, $\neg(i\oplus j)$ implies that there is an event in the interval $[t_i, t_j)$. Since there are no events during adapting phases, it follows that $i=j-2$, contradiction. Claim proved.

We claim that there can be at most \hlc{2}{$2k + d_{\FP}$} testing phases (and hence at most as many adapting phases), including at most \hlc{2}{$k + d_{\FP}$} event-free testing phases. Indeed, in the first round of each testing phase $j$ {\tt classifier} generates a ``positive", and in the end of the phase we generate a label
	$l\in\{ \texttt{true}, \texttt{false}\}$.
We examine each case separately: (i) if $l = \texttt{false}$ then {\tt classifier} receives feedback, so there can be at most \hlc{2}{$d_{\FP}$} such phases $j$, (ii) if $l = \texttt{true}$ then an event has occurred in phase $j$ or $j-2$, so there can be at most $2k$ such phases $j$, of which at most $k$ phases can be event-free. Claim proved.

To obtain the regret bound~\refeq{eq:meta-algorithm}, note that regret in each event-free phase of length $t$ is $R_0(t)$, see the second remark after Theorem~\ref{thm:meta-algorithm} for details.
\end{proof}

\OMIT{testing phase is at most $R_0(L)$ if the phase is event-free, and at most $L$ otherwise, and in each adapting phase (since there are no events) regret is at most $R_0(T)$.}

\section{Safe Classifier}
\label{sec:safe}

\newcommand{\safecl}{{\tt SafeCl}}


\hlc{2}{In this section, we show how sa}fe classifiers with low \AlexDiam{} can be constructed for specific concept classes. Recall that a classifier is called safe if (assuming it inputs only correctly labeled samples) it never outputs a false negative, and the definition of \AlexDiam{}, motivated by the specification of the $\ucbc$ algorithm, essentially assumes that all labeled samples correspond to false positives. 

\hlc{2}{We first describe a generic class}ifier, called \safecl{}, that is safe for \emph{any} concept class $\CF$, and bound its \AlexDiam{} using a certain property of $\CF$. In the event that the concept class is all $d$-dimensional axis-parallel hyper-rectangles with margin $1/\delta$, we show that this bound is proportional to $d/\delta$.  And in the event that the concept class is all $d$-dimensional hyperplanes with margin $\delta$, we show that this bound is exponential in $d$.  Unfortunately, the exponential dependence cannot be improved, as we will see in Section \ref{lowerb}.

The classifier \safecl{} is defined as follows. 

\framebox[\columnwidth]{
\parbox{.95\columnwidth}{
\safecl{} classifies a given unlabeled context $x$ as negative if and only if there exists no concept $f \in \CF$ such that $f(x) = +1$ and \hlc{2}{$f(x') = -1$ for each} {\tt false}-labeled example \hlc{2}{$x'$} received so far.
}}

It is easy to see that this classifier is indeed safe. Moreover, we bound its \AlexDiam{} in terms of the following property of the concept class $\CF$:

\begin{defn}\label{def:diam}
The \emph{diameter} of $\CF$, denoted $d_\CF$, is equal to the length of the longest sequence $x_1, \ldots, x_m \in \CX$ such that for each $t = 1, \ldots, m$ there exists a concept $f \in \CF$ with the following property: $f(x_t) = +1$, and $f(x_s) = -1$ for all $s<t$.
\end{defn}

\begin{claim}
\safecl{} is safe, and its \AlexDiam{} is at most $d_\CF$.
\end{claim}
\begin{proof}
Assume all {\tt false}-labeled examples input by \safecl{} are correctly labeled. Suppose \safecl{} outputs a false negative, with concept $f\in\CF$ and unlabeled sample $x$. Then $f(x) = +1$ and \hlc{2}{$f(x')=-1$ for each} {\tt false}-labeled example \hlc{2}{$x'$} received so far. But by definition of  \safecl{} such concept does not exist, contradiction. Therefore, \safecl{} is safe. Regarding the \AlexDiam, consider the prediction game in Definition~\ref{def:FP-complexity}. Any sequence $x_1, \ldots, x_m$ of false positives output by \safecl{} satisfies the property in Definition~\ref{def:diam}, so $m \leq d_\CF$.
\end{proof}

\OMIT{ 
\begin{defn} Define the \emph{safe function} $S_\CF : 2^\CX \goes 2^\CX$ of $\CF$ as follows: $x \in S_\CF(N)$ if and only if there is no concept $f \in \CF$ such that: $f(y) = -1$ for all $y \in N$ and $f(x) = +1$.
The \emph{diameter} of $\CF$, denoted $d_\CF$, is equal to the length of the longest sequence $x_1, \ldots, x_m \in \CX$ such that $x_t \notin S_\CF(\{x_1, \ldots, x_{t-1}\})$ for all $t = 1, \ldots, m$.
\end{defn}
} 

By using \safecl{} as our classifier, we introduce $d_\CF$ into the regret bound of ${\tt bwc}$, and this quantity can be large. However, in Section \ref{lowerb} we show that the regret of \emph{any} algorithm must depend on $d_\CF$, unless it depends strongly on the number of rounds $T$.

\OMIT{ 
Moreover, we give examples of common concept classes with efficiently computable safe functions. For example, if $\CF$ is the space of hyperplanes with ``margin'' at least $\delta$ (probably the most commonly-used concept class in machine learning), then $S_\CF(N)$ is the convex hull of the examples in $N$, extended in all directions by a $\delta$.}

\newcommand{\hyp}{\CF_{\texttt{HYP($d,\delta$)}}}
\newcommand{\apr}{\CF_{\texttt{APR($d,\delta$)}}}
\newcommand{\mynull}{\texttt{null}}
\newcommand{\myhull}{\texttt{Co}}

Below we give examples of common concept classes with efficiently computable safe functions, and prove bounds on their diameter.  Recall that for a given universe $\CX$ of examples, a concept is a function
	$f : \mathcal{X} \to \{-1, +1, \mynull\}$,
where the \mynull{} value refers to the examples that are not feasible under a given concept \hlc{2}{(i.e., if $f$ is the true concept, then we will never observe an example $x$ such that $f(x) = \mynull$).}

In what follows, for each $N\subset \CX$ define $S_\CF(N)\subset \CX$ as the set of all $x\in \CX$ for which there is no concept $f \in \CF$ such that $f(x) = +1$ and \hlc{2}{$f(x') = -1$ for each $x' \in N$}. Note that \safecl{} outputs a negative prediction on $x$ if and only if $x\in S_\CF(N)$, where $N$ is the set of {\tt false}-labeled samples received so far. Likewise, in Definition~\ref{def:diam} the sequence $\{x_t\}$ satisfies
	$x_t \notin S_\CF(\{x_1, \ldots, x_{t-1}\})$
for each $t$.

For convenience, define a ``$\delta$-ball" around a set $S\subset\bbR^d$ in the $d$-dimensional $L_p$-norm as
\begin{align*}
	\bbB^d_p(S,\delta) \triangleq
		\{x\in \bbR^d:\, L_p(x,S) \le \delta \},
		\;\text{where}\;
L_p(x,S) \triangleq \textstyle{\min_{y\in S}}\, \norm{x-y}_p.
\end{align*}
Here $L_p(x,S)$ is the $L_p$-norm distance between a point $x$ and a set $S$.

\subsection{Axis-parallel rectangles with margin $\delta$}

\hlc{2}{One very simple concept is an axis-parallel hyper-rectangle. This type of concept can be used to test whether any one of several features is outside of its `normal' range. This is a particularly well-suited concept class for predicting events that may affect a search engine query, since these events are typically preceded by a large change in some statistic related to the query, such as its volume or abandonment rate.}

Fix the dimension $d$, and let $\CX \subseteq \bbR^d$ be the $d$-dimensional $L_\infty$-norm unit ball around the origin. A \emph{$d$-rectangle} in $\bbR^d$ is the cross-product of $d$ non-empty intervals in $\bbR$. Given $\delta > 0$ and a $d$-rectangle $R$, define a function
	$f_{R,\delta}:\CX \goes \{-1, +1, \mynull\}$
as follows: $f_{R,\delta}(x)$ equals $+1$ if $L_\infty(x,R) \ge \delta$; it equals $-1$ if $x \in R$, and it equals $\mynull$ otherwise \hlc{2}{(note that the margin $\delta$ only applies only outside of $R$)}. The concept class of $d$-dimensional \emph{axis-parallel rectangles with margin $\delta$} is defined as
	$\apr = \{f_{R,\delta}:\; \text{all $d$-rectangles $R$} \} $.

We bound the diameter of \hlc{2}{$\apr$} as follows.

\begin{claim}
If $\CF = \apr$, then $d_\CF \le O(d/\delta)$.
\end{claim}
\begin{proof} Consider a sequence $x_1, \ldots, x_m \in \CX$ such that $x_t \notin S_\CF(\{x_1, \ldots, x_{t-1}\})$ for all $1 \le t \le m$. Let $R_t$ be the $\delta$-ball in $L_\infty$ around the smallest $d$-rectangle containing $x_1, \ldots, x_t$. By definition of the sequence, at least one of the one-dimensional intervals defining $R_{t+1}$ must be $\delta$ larger than the same interval in $R_t$. Since $\norm{x_t}_\infty \le 1$, $m \le O(d/\delta)$.
\end{proof}

\hlc{2}{Clearly, for the concept class $\apr$, the classifier \safecl{} simply maintains the smallest $d$-dimensional rectangle $R(N)$ containing the set of all previously {\tt false}-labeled examples $N$, and classifies a new example $x$ as negative if and only if $x$ lies within $\delta$ (measured in $L_\infty$-norm) of $R(N)$. In other words} 

\begin{center}
\framebox[.90\columnwidth][c]{
\parbox{.80\columnwidth}{
\safecl{} on $\apr$: classify $x\in \CX$ as negative $\iff$ $x\in \bbB^d_\infty(R(N),\,\delta)$,\\ where $N$ is the set of all {\tt false}-labeled examples received so far.
}}
\end{center}

\subsection{Hyperplanes with margin $\delta$}

\hlc{2}{Hyperplanes are perhaps the most widely-used concept in classification problems.} Fix the dimension $d$, and let $\CX \subseteq \bbR^d$ be the $d$-dimensional $L_2$-norm unit ball around the origin. Given $u,w\in \bbR^d$ and $\delta > 0$, define a function
	$f_{u,w,\delta} : \CX \goes \{-1, +1, \mynull\}$
as follows: $f_{u,w,\delta}(x)$ equals $+1$ if $w \cdot (x+u) \ge \delta$, it equals $-1$ if $w \cdot (x+u) < -\delta$, and it equals $\mynull$ otherwise. Here $w$ is the unit normal of the hyperplane, and $u$ is the shift vector. The concept class of $d$-dimensional \emph{hyperplanes with margin $\delta$} is defined as
\begin{align*}
	\hyp = \{ f_{u,w,\delta}:\;
		u,w\in\bbR^d,\, \norm{u}_2 \le 1, \norm{w}_2=1 \}.
\end{align*}
We bound the diameter of $\hyp$ as follows:

\begin{claim}
If $\CF = \hyp$, then $d_\CF \le (1+\tfrac{1}{\delta})^d$.
\end{claim}
\begin{proof}
Consider a sequence $x_1, \ldots, x_m \in \CX$ such that $x_t \notin S_\CF(\{x_1, \ldots, x_{t-1}\})$ for all $1 \le t \le m$, as in Definition~\ref{def:diam}. Then for each $s$ and $t$ such that
	$1 \le s < t \le m$
there exist $u,w \in \bbR^d$ such that $\norm{u}_2 \le 1$, $\norm{w}_2=1$,
	$w\cdot (x_t+u) \geq \delta$
and
	$w\cdot (x_s+u) < -\delta$.
By H\"{o}lder's inequality, it follows that
\begin{align}\label{eq:balls-dont-overlap}
\norm{x_t-x_s}_2
	= \norm{w}_2\, \norm{x_t-x_s}_2
	\geq w\cdot (x_t-x_s)
	> 2\delta.
\end{align}
Now, place an $L_2$-ball of radius $\delta$ around each point $x_t$. By~\refeq{eq:balls-dont-overlap}, none of these balls can intersect. A radius-$r$ ball in $d$ dimensions has volume $C_d\, r^d$, where $C_d$ is a constant that depends only on $d$. Thus the total volume of the balls is $m\, C_d\, \delta^d$. On the other hand, $\norm{x_t}_2 \le 1$ for each $t$, so each of these balls lies in the radius-$(1+\delta)$ ball around the origin, so their total volume is at most $C_d\, (1+\delta)^d$. It follows that
	$m\leq (1+\tfrac{1}{\delta})^d$.
\end{proof}

We now show that there is a computationally efficient way to implement the classifier \safecl{} for hypothesis class $\hyp$. \hlc{2}{Specifically, we show that the classifier $\safecl{}$ simply maintains the convex hull $\myhull(N)$ of all previously {\tt false}-labeled examples $N$, classifies a new example $x$ as negative if and only if $x$ lies within $2\delta$ (measured in $L_2$-norm) of $\myhull(N)$. In other words}

\begin{center}
\framebox[.90\columnwidth][c]{
\parbox{.80\columnwidth}{
\safecl{} on $\hyp$: classify $x\in \CX$ as negative $\iff$ $x\in \bbB^d_2(\myhull(N),\,2\delta)$,\\ where $N$ is the set of all {\tt false}-labeled examples received so far.
}}
\end{center}

\begin{claim} If $\CF = \hyp$ and $N \subset \CX$ then $S_\CF(N) = \CX \cap \bbB^d_2(\myhull(N),\,2\delta)$, where $\myhull(N)$ is the convex hull of $N$.\end{claim}
\begin{proof} Fix $x_t \in \CX$. We divide the \hlc{2}{proof} into two parts. \hlc{2}{First, we show that if $x_t$ is contained in the $2\delta$-ball around $\myhull(N)$, then no hyperplane in $\CF$ can separate $x_t$ from $N$. Next, we show that if $x_t$ is outside the $2\delta$-ball around $\myhull(N)$, then at least one hyperplane in $\CF$ separates $x_t$ from $N$. More precisely, we prove that}
\begin{enumerate}
	\item[(i)] If $x_t \in \bbB^d_2(\myhull(N),\,2\delta)$ then there does not exist $f \in \hyp$ such that \hlc{2}{$f(x_s) = -1$} for all $x_s \in N$ and $f(x_t) = +1$.
	\item[(ii)] If $x_t \notin \bbB^d_2(\myhull(N),\,2\delta)$ then there exists $f \in \hyp$ such that $f(x_s) = -1$ for all $x_s \in N$ and $f(x_t) = +1$.
\end{enumerate}


Proof of (i): Suppose for contradiction that there exist $u, w \in \bbR^d$, with $\norm{u}_2 \le 1$ and $\norm{w}_2 = 1$, such that $w \cdot (x_s + u) < -\delta$ for all $x_s \in N$ and $w \cdot (x_t + u) \ge \delta$. 

Choose $x^* \in \myhull(N)$ so that $\norm{x_t - x^*}_2 = L_2(x_t, \myhull(N))$, i.e. $x^*$ is a closest point in $\myhull(N)$ to $x_t$ (we know $x^*$ exists because $\myhull(N)$ is closed). Since $x_t \in \bbB^d_2(\myhull(N),\,2\delta)$, we have that $\norm{x_t - x^*}_2 \le 2\delta$.

We know that $w \cdot (x^* + u) < -\delta$ because $x^*$ is a convex combination of the examples in $N$. Therefore, by the intermediate value theorem, there exists $x' \in \CX$ and $\theta \in [0, 1]$ such that $x' = (1 - \theta) x_t + \theta x^*$ and $w \cdot (x' + u) = 0$.


Some algebra shows that $\norm{x_t - x'}_2 = \theta \norm{x_t - x^*}_2$ and $\norm{x' - x^*}_2 = (1-\theta)\norm{x_t - x^*}_2$. Adding these equations yields
$$
\norm{x_t - x'}_2 + \norm{x' - x^*}_2 = \norm{x_t - x^*}_2
$$


Because $w \cdot (x' + u) = 0$, by H\"{o}lder's inequality we have
$$
\norm{x_t - x'}_2 = \norm{w}_2 \norm{x_t - x'}_2 \ge w \cdot (x_t -  x') = w \cdot (x_t + u) - w \cdot (x' + u) \ge \delta
$$
and
$$
\norm{x' - x^*}_2 = \norm{w}_2 \norm{x' - x^*}_2 \ge w \cdot (x' -  x^*) = w \cdot (x' + u) - w \cdot (x^* + u) > \delta
$$
which implies $\norm{x_t - x^*}_2 > 2\delta$, which is a contradiction.

\newcommand{\interior}{\textrm{\tt int}}

Proof of (ii): We will use the well-known \emph{separating hyperplane theorem} \cite{Rock}: If nonempty convex sets $X, Y \in \bbR^d$ do not intersect, then there exist $a \in \bbR^d \setminus \{0\}$ and $b \in \bbR$ such that 
\be
a \cdot x \ge b \textrm{ for all }x \in X\textrm{ and }a \cdot y \le b\textrm{ for all }y \in Y \label{eq:sephyp}
\ee
Since $x_t \notin B^d_2(\myhull(N), 2\delta)$ there must exist $\eps > 0$ such that the sets $X = B^d_2(\{x_t\}, \delta)$ and $Y = B^d_2(\myhull(N), \delta + \eps)$ do not intersect. For these choices for $X$ and $Y$, let us fix $a \in \bbR^d \setminus \{0\}$ and $b \in \bbR$ that satisfy \eqref{eq:sephyp}.

Note that $x_t + z \in X$ for all $z \in \bbR^d$ such that $\norm{z}_2 \le \delta$. Also note that $x_s + z \in Y$ for all $x_s \in N$ and $z \in \bbR^d$ such that $\norm{z}_2 \le \delta + \eps$. So by \eqref{eq:sephyp} we have
$$
a \cdot \left(x_t - \delta\frac{a}{\norm{a}_2}\right) \ge b\textrm{ and }a \cdot \left(x_s + (\delta + \eps)\frac{a}{\norm{a}_2}\right) \le b\textrm{ for all }x_s \in N
$$
Letting $w = \frac{a}{\norm{a}_2}$ and rearranging we have
\be
w \cdot x_t \ge \frac{b}{\norm{a}_2} + \delta\textrm{ and }w \cdot x_s \le \frac{b}{\norm{a}_2} - (\delta + \eps)\textrm{ for all }x_s \in N\label{eq:without_shift}
\ee
Since $\norm{w}_2 = 1$ and $\norm{x}_2 \le 1$ for all $x \in \CX$, it follows from \eqref{eq:without_shift} that $\left|\frac{b}{\norm{a}_2}\right| \le 1$. Thus there exists $u \in \bbR^d$ such that $\norm{u}_2 \le 1$ and $w \cdot u = -\frac{b}{\norm{a}_2}$. It now follows that
$$
w \cdot (x_t + u) \ge \delta\textrm{ and }w \cdot (x_s + u) \le -\delta - \eps\textrm{ for all }x_s \in N
$$
So the function $f_{u,w,\delta} \in \hyp$ satisfies the claim.
\end{proof}


\section{Testable Bandit Algorithms}
\label{sec:UCB}

In this section we will consider the stochastic $n$-armed bandit problem. We are looking for $(L,\eps)$-testable algorithms with low regret. The $L$ will need to be sufficiently large, on the order of  $\Omega(n\eps^{-2})$.

A natural candidate would be algorithm $\ucb$ from~\cite{bandits-ucb1} which does very well on event-free regret:
\begin{align}\label{eq:regret-UCB1}
    R_0(L) \leq O(\min(\tfrac{n}{\Delta} \log L, \; \sqrt{nL\log L})).
\end{align}
Unfortunately, $\ucb$ does not immediately provide a way to define the $t$-th round best guess $(G^+,G^-)$ so as to guarantee $(L,\eps)$-testability. One simple fix is to choose an arm at random in each of the first $L$ rounds, use these samples to form the best guess, in a straightforward way, and then run $\ucb$. However, in the first $L$ rounds this algorithm incurs regret of $\Omega(L)$, which is very suboptimal \hlc{1}{compared to $R_0(L)$ from}~\refeq{eq:regret-UCB1}.

In this section, we develop an algorithm which has the same regret bound as $\ucb$, and is $(L,\eps)$-testable. We state this result more generally, in terms of estimating expected payoffs; we believe it may be of independent interest. The $(L,\eps)$-testability is then an easy corollary.

Since our analysis in this section is for the event-free setting, we can drop the subscript $t$ from much of our notation. Let $p(u)$ denote the (time-invariant) expected payoff of arm $u$. Let $p^* = \max_u p(u)$, and let
	$\Delta(u) = p^* - p(u)$
be the ``suboptimality'' of arm $u$. For round $t$, let $\mu_t(u)$ be the sample average of arm $u$, and let $n_t(u)$ be the number of times arm $u$ has been played.

We will use a slightly modified algorithm $\ucb$ from~\cite{bandits-ucb1}, with a significantly extended analysis. Recall that in each round $t$ algorithm $\ucb$ chooses an arm $u$ with the highest \emph{index}
	$I_t(u) = \mu_t(u) + r_t(u)$,
where
	$r_t(u) =  \sqrt{8 \log(t)/ n_t(u)}$
is a term that we'll call the \emph{confidence radius} whose meaning is that
    $|p(u) - \mu_t(u)| \leq r_t(u)$
with high probability. For our purposes here it is instructive to re-write the index as
	$I_t(u) = \mu_t(u) + \alpha\, r_t(u)$
for some parameter $\alpha$. Also, to better bound the early failure probability we will re-define
the confidence radius as
	$r_t(u) =  \sqrt{8\log(t_0+t)/ n_t(u)}$
for some parameter $t_0$. We will denote this parameterized version by $\ucb(\alpha,t_0)$. 

The original regret analysis of $\ucb$ in~\cite{bandits-ucb1} carries over \hlc{1}{to $\ucb(\alpha,t_0)$ so as to guarantee event-free regret}~\refeq{eq:regret-UCB1}; we omit the details.

Our contribution concerns estimating the $\Delta(u)$'s. We estimate the maximal expected reward $p^*$ via the sample average of an arm that has been played most often. More precisely, in order to bound the failure probability we consider an arm that has been played most often \emph{in the last $t/2$ rounds}. For a given round $t$ let $v_t$ be one such arm (ties broken arbitrarily), and let
	$\Delta_t(u) = \mu_t(v_t) - \mu_t(u)$
will be our estimate of $\Delta(u)$. This estimate (and the provable guarantee thereon) is the main technical contribution of this section.

We obtain an $(L,\eps)$-testable algorithm from $\ucb(6,T)$, where $T$ is the time horizon, by defining the $t$-th round guess as
\begin{align}\label{eq:guess}
(G^+, G^-) =(	
		\{v: \Delta_t(v) \leq \eps/4\},\;
		\{v: \Delta_t(v) > \eps/2\}).
\end{align}
The pseudocode is in Algorithm~\ref{alg:UCBhack}.

\begin{algorithm}[t]
\caption{The $(L,\eps)$-testable bandit algorithm with low regret.}
\label{alg:UCBhack}
\begin{algorithmic}[1]
\STATE {\bf Given:} Time horizon $T$, parameter $\eps\in(0,1)$.
\FOR{all arms $u$}
\STATE $n(u)\leftarrow 0$, $x(u) \leftarrow 0$, $\mu(u)\leftarrow 0$
	~~~\COMMENT{\#samples, total reward, sample average}
\ENDFOR
\FOR{rounds $t = 1,2,\, \ldots\, , T$~}
\STATE Pick arm $u$ with the maximal index
	$I(u) = \mu(u) + 12\sqrt{\tfrac{2\,\log(t+T)}{1+n(u)}}$.
\STATE Observe payoff $x$, update
	$n(u) \leftarrow n(u)+1$,
	~~$x(u) \leftarrow x(u)+x$,
	~~$\mu(u) \leftarrow x(u)/n(u)$.
\STATE \COMMENT{ Form the $t$-th round guess }
\STATE $v^* \leftarrow$ arm played most often in the last $t/2$ rounds.
\FOR{all arms $v$}
\STATE $\widehat{\Delta}(v) \leftarrow \mu(v^*) - \mu(v) $
	~~~~~~~\COMMENT{the $t$-th round estimate of $\Delta(v)$}
\ENDFOR
\STATE Output
	$(G^+, G^-) =\left(	
		\{v: \widehat{\Delta}(v) \leq \eps/4\},\;
		\{v: \widehat{\Delta}(v) > \eps/2\}
	\right)$.
\ENDFOR
\end{algorithmic}
\end{algorithm}

Let us pass to the provable guarantees. We express the ``quality" of our estimate $\Delta_t$ as follows:

\OMIT{For notational convenience, $\WHP{t}$ will mean \emph{with probability at least $1-t^{-3}$}.}

\begin{thm}\label{thm:estimate-DeltaT}
Consider the stochastic $n$-armed bandits problem. Suppose algorithm $\ucb(6,t_0)$ has been played for $t$ steps, and $t+t_0\geq 32$. Then with probability at least $1-(t_0+t)^{-2}$ for any arm $u$ we have
\be
|\Delta(u) - \Delta_t(u)| < \tfrac14 \Delta(u) + \delta(t) \label{eq:DeltaT-thm}
\ee
where $\delta(t) = O(\sqrt{\tfrac{n}{t} \log (t+t_0) })$.
\end{thm}

\begin{note}{Remark.}
Either we know that $\Delta(u)$ is small, or we can approximate it up to a constant factor. Specifically, if
	$\delta(t) < \tfrac12\, \Delta_t(u)$
then
	$\Delta(u) \leq 2\,\Delta_t(u) \leq 5\, \Delta(u)$
else
	$\Delta(u) \leq 4 \delta(t)$.
\end{note}

\begin{proof}
Fix round $t$, let $v^* = v_t$ and let $s$ be the last round this arm has been played before round $t$. Recall that $s\geq t/2$ by definition of $v_t$.  Since by pigeonhole principle $n_t(v^*) \geq \tfrac{t}{2 n}$, it follows that $r_t(v) \leq O(\delta)$ where
	$\delta = \sqrt{\tfrac{n}{t} \log (t+ t_0)}$.
It is easy to see that
\begin{align*}
 r_s(v^*) \leq 2\, r_{s+1}(v^*)
	\leq 2\, r_t(v^*) = O(\delta).
\end{align*}

Then with probability at least $1-(t_0+t)^{-2}$ for any arm $u$ we have
\begin{align}\label{eq:thm-estimate-DeltaT-proof}
	p(v^*) + O(\delta)
	\geq  p(v^*) + 7 r_s(v^*)
	\geq I_s(v^*) \geq I_s(u) \geq p(u) + 5 r_s(u).
\end{align}
If $u^*$ is the arm with maximal expected reward, then plugging $u = u^*$ into
~\refeq{eq:thm-estimate-DeltaT-proof} gives
	$\Delta(v^*) \leq O(\delta)$.

We claim that~\refeq{eq:thm-estimate-DeltaT-proof} implies
	$r_t(u) \leq \tfrac14\,\Delta(u) + O(\delta)$.
Indeed, we can re-write~\refeq{eq:thm-estimate-DeltaT-proof} as
$$ 5 r_s(u) \leq p(v^*) - p(u) + O(\delta) \leq \Delta(u) + O(\delta).
$$
The claim follows since
	$r_t(u) \leq r_s(u)\, \log(t_0+t)/ \log(t_0+s) \leq \tfrac54\, r_s(u)$.

Now we are ready for the final calculation. Let $p^*$ be the maximal expected reward. Then
\begin{align*}
|\Delta(u) - \Delta_t(u)|
	&= | p^* - p(u) - \mu_t(v^*) + \mu_t(u) | \\
	&= | (p^* - p(v^*)) + (p(v^*) - \mu_t(v^*)) + (\mu_t(u) - p(u)) | \\
	&\leq \Delta(v^*) + |p(v^*) - \mu_t(v^*)| + |\mu_t(u) - p(u)| \\
	&\leq \Delta(v^*) + r_t(v^*) + r_t(u^*)
	\leq  \tfrac14\, \Delta(u) + O(\delta). \qedhere
\end{align*}
\end{proof}

Finally, let us prove that  Algorithm~\ref{alg:UCBhack}
is $(L,\eps)$-testable  as long as
	$L\geq \Omega(\tfrac{n}{\eps^2} \log T)$.

\begin{thm}
Consider algorithm $\ucb(6,T)$ where $T$ is the time horizon and the $t$-th round guess is given by~\refeq{eq:guess}. Assume that $\delta(L)\leq \eps/4$, where  $\delta(t)$ is from~\refeq{eq:DeltaT-thm}. Then the algorithm is $(L,\eps)$-testable.
\end{thm}

\begin{proof}
If $u$ is an optimal arm, then $\Delta(u)=0$, so by~\refeq{eq:DeltaT-thm} we have
	$\Delta_t(u) \leq \delta(t) \leq \eps/4$.
If $\Delta(u)\geq \eps$ then by~\refeq{eq:DeltaT-thm} we have
	$\Delta_t(u) \geq \Delta(u)/2 \geq \eps/2$.
\end{proof}

\section{Upper and Lower Bounds}
\label{lowerb}

Plugging the classifier from Section~\ref{sec:safe} and the bandit algorithm from Section~\ref{sec:UCB} into the meta-algorithm from Section~\ref{sec:BWC}, we obtain the following numerical guarantee.

\begin{thm}\label{thm:algorithm}
Consider an instance $\mathcal{S}$ of the eventful bandit problem with number of rounds $T$, $n$ arms, $k$ events, minimum shift $\eps_S$, minimum suboptimality $\Delta$, and concept class diameter $d_\CF$. Assume that any two events are at least $2L$ rounds apart, where $L = \Theta(\tfrac{n}{\eps_S^2} \log T)$. Consider the $\ucbc$ algorithm with parameter $L$ and components {\tt classifier} and {\tt bandit} as presented, respectively, in Section~\ref{sec:safe} and Section~\ref{sec:UCB}. Then the regret of~ $\ucbc$ is
\begin{align*}
	R_{\tt BWC}(T) \leq  \left( (3k+2d_\CF)\tfrac{n}{\Delta} + k \tfrac{n}{\eps_S^2} \right)(\log T).
\end{align*}
\end{thm}

While the linear dependence on $n$ in this bound may seem large, note that without additional assumptions, regret must be linear in $n$, since each arm must be pulled at least once. In an actual search engine application, the arms can be restricted to, say, the top ten results that match the query.

We now state two lower bounds about eventful bandit problems. Theorem \ref{thm:LB-contexIgnoring} shows that in order to achieve regret that is logarithmic in the number of rounds, a context-aware algorithm is necessary, assuming there is at least one event. Incidentally, this lowerbound can be easily extended to prove that, in our model, \emph{no} algorithm can achieve logarithmic regret when an event oracle $f$ is not contained in the concept class $\CF$.


\begin{thm} \label{thm:LB-contexIgnoring}
Consider the eventful bandit problem with number of rounds $T$, two arms, minimum shift $\eps_S$ and minimum suboptimality $\Delta$, where $\eps_S = \Delta = \eps$, for an arbitrary $\eps \in (0,\tfrac12)$. For any context-ignoring bandit algorithm $\A$, there exists a problem instance with a single event such that regret $R_\A(T) \ge \Omega(\eps\sqrt{T})$.
\end{thm}

\begin{proof}
For simplicity, assume that $N = \sqrt{T}$ is an integer. Define problem instances $\mathcal{I}_i$, $0\leq i\leq N$ as follows. In each of these instances, the $T$ rounds are partitioned into $N$ phases, each of length $N$. There are two arms, call them $y$ and $z$. Set $p_t(y) = \tfrac12$ for all $t$. For the problem instance $\mathcal{I}_0$, $p_t(z)=\tfrac12-\eps$ for all $t$. For problem instances $\mathcal{I}_i$, $i\geq 1$ set $p_t(z) = \tfrac12 - \eps$ in all phases $j<i$, and $p_t(z) = \tfrac12 + \eps$ in all phases $j\geq i$. (Thus, in each instance $\mathcal{I}_i$ there is a single event that occurs in the first round of phase $i$.)

Now, let $q_i$ be the probability that on problem instance $\mathcal{I}_0$, arm $z$ is chosen by algorithm $\A$ at least once during phase $i$. If $q_i \geq \tfrac12$ for each phase $i$, then on the problem instance $\mathcal{I}_0$ each phase $i$ contributes at least $\eps/2$ to regret, so the total regret is at least $\eps N/2$. Otherwise, $q_i<\tfrac12$ for some $i$. Since instances $\mathcal{I}_0$ and $\mathcal{I}_i$ coincide on the first $i-1$ phases, algorithm $\A$ behaves the same way on both instances up to the end of phase $i-1$. Moreover, $\A$ behaves the same way on both instances throughout phase $i$ assuming that it never plays arm $z$ during that phase. Therefore with probability $1-q_i$ its regret on instance $\mathcal{I}_i$ due to phase $i$ alone is $\eps$ per each round in this phase; so the total regret is at least $\eps N/2$.
\end{proof}

Theorem \ref{thm:LB-eventful} proves that in Theorem~\ref{thm:algorithm}, linear dependence on $k+d_\CF$ is essentially unavoidable. If we desire a regret bound that has logarithmic dependence on the number of rounds, then a linear dependence on $k+d_\CF$ is necessary.

\begin{thm} \label{thm:LB-eventful}
Consider the eventful bandit problem with number of rounds $T$ and concept class diameter $d_\CF$. Let $\A$ be an eventful bandit algorithm. 
\begin{enumerate}
	\item[(i)] There exists a problem instance with $n$ arms, $k$ events, minimum shift $\eps_S$, minimum suboptimality $\Delta$, where $\eps_S = \Delta = \eps$, for arbitrary $k\geq 1$, $n\geq 3$, and $\eps\in (0,\tfrac14)$, such that $R_\A(T) \ge \Omega(k\, \tfrac{n}{\eps})\, \log (T/k)$.

	\item[(ii)] There exists a problem instance with two arms, a single event, minimum shift $\Theta(1)$ and minimum suboptimality $\Theta(1)$ such that regret
    	$R_\A(T) \ge \Omega(T^{1/3})$ or $R_\A(T) \ge \Omega(d_\CF \log T)$.
\end{enumerate}
\end{thm}

\begin{proof}
For part (i), construct the family of problem instances as follows. In each instance, there are $k$ phases of length $T/k$ each. For each phase $i$, one arm, call it $y_i$, has payoff $p_t(y_i) = \tfrac12+\eps$, and all other arms $y$ have payoff $p_t(y)=\tfrac12-\eps$. We have one problem instance for each sequence $\{y_i\}$ such that $y_i \neq y_{i+1}$ for each $i$. Note that there is an event in the first round of each phase; without loss of generality let us assume that this is known to the algorithm. Then in each phase $i\geq 1$ the algorithm (essentially) needs to solve a fresh instance of the stochastic bandit problem on $n-1$ arms with time horizon $T/k$ and payoffs $\tfrac12\pm \eps$, which implies regret $\Omega(\tfrac{n}{\eps}) \log (T/k)$~\cite{Lai-Robbins-85,bandits-ucb1}. We omit the easy formal details.

For part (ii), partition the $T$ rounds into $N = \min(d_\CF, T^{1/3})$ phases, each of length at least $T^{2/3}$. We define problem instances $\mathcal{I}_i, 0 \le i \le N$, in a similar way as in Theorem~\ref{thm:LB-contexIgnoring}. There are two arms, $y$ and $z$. Set $p_t(y) = \frac{1}{2}$ for all $t$. For problem instance $\mathcal{I}_0$, set $p_t(z) = \frac{1}{1 + e^{1/3}}$. In problem instance $\mathcal{I}_i$, for $i \ge 1$, set $p_t(z) = \frac{1}{1 + e^{1/3}}$ in all phases $j<i$, and $p_t(z) = \frac{e^{1/3}}{1 + e^{1/3}}$ in all phases $j\geq i$. Note that $\frac{1}{1 + e^{1/3}} < \frac{1}{2} < \frac{e^{1/3}}{1+e^{1/3}}$.

In Appendix~\ref{app:LB}, we show how to define the context sequence $\{x_t\}$ in a way consistent with all our assumptions, in such a way that the contexts for problem instances $\mathcal{I}_0$ and $\mathcal{I}_i$ agree in the first $i$ phases. The idea is that for both problem instances, the first round of each phase $j<i$ triggers a false positive; this is possible since (essentially) we are allowed $d_\CF$ false positives.

The rest of the proof involves calculations similar to those in the proof of Theorem~\ref{thm:LB-contexIgnoring}. First, suppose $d_\CF \ge T^{1/3}$. Define $q_i$ as in the proof of Theorem~\ref{thm:LB-contexIgnoring}. If $q_i \geq \tfrac12$ for each phase $j$, then for the problem instance $\mathcal{I}_0$ we have $R_\A(T) \ge \Omega(T^{1/3})$. Otherwise, let $i$ be such that $q_i < \tfrac12$. By our construction, with probability $1-q_i$ algorithm $\A$ behaves identically on instances $\mathcal{I}_0$ and $\mathcal{I}_i$ through the first $i$ phases. Thus, on instance $\mathcal{I}_i$ in phase $i$ alone it incurs regret $\Omega(1)$ per each round of the phase, for a total of  $R_\A(T) \geq \Omega(T^{2/3})$.

Next, suppose $d_\CF < T^{1/3}$. Let $q_{i,j}$ be the probability that for problem instance $\mathcal{I}_j$, arm $z$ is chosen by $\A$ at least $\log T$ times during phase $i$. If $q_{i,0} \ge \frac{1}{2}$ for each phase $i$, then $R_\A(T) \ge \Omega(d_\CF \log T)$ on problem instance $\mathcal{I}_0$. Otherwise, let $i$ be such that $q_{i,0} < \frac{1}{2}$. In Appendix~\ref{app:LB}, we give a calculation that shows that
	$(1-q_{i,i}) \geq T^{-1/3}(1-q_{i,0})  $,
which implies that $R_\A(T) \ge \Omega(\frac{1}{T^{1/3}}) (T^{2/3} - \log T) \ge \Omega(T^{1/3})$
on problem instance $\mathcal{I}_i$.
\end{proof}

\section{Experiments}
To truly demonstrate the benefits of $\ucbc$ requires real-time
manipulation of search results.  Since we did not have the means to
deploy a system that monitors click/skip activity and correspondingly
alters search results with live users, we describe a
collection of experiments on synthetically generated data.

We begin with a head-to-head comparison of $\ucbc$ versus a baseline $\ucb$
algorithm and show that $\ucbc$'s performance improves substantially
upon $\ucb$.  Next, we compare the performance of these algorithms as we
vary the fraction of intent-shifting queries: as the fraction increases,
$\ucbc$'s performance improves even further upon prior approaches.
Finally, we compare the performance as we vary the number of features.
While our theoretical results suggest that regret grows with the
number of features in the context space, in our experiments,
we surprisingly find that $\ucbc$ is robust to higher dimensional feature
spaces.

{\noindent \bf Setup:}
We synthetically generate data as follows.  We assume that there are
100 queries where the total number of times these queries are posed is
3M.  Each query has five search results for a user to select
from.  If a
query does not experience any events --- i.e., it is not ``intent-shifting'' --- then the optimal search result is fixed over time; otherwise the optimal search result may change.
Only
10\% of the queries are intent-shifting, with at most 10 events per such query.
Due to the random nature with which data
is generated, regret is reported as an average over 10 runs.  The event oracle is
an axis-parallel rectangle
anchored at the origin, where points inside the box are negative and
points outside the box are positive.  Thus, if there are two features, say
query volume and query abandonment rate, an event occurs if and only if
both the volume and abandonment rate exceed certain thresholds.

{\noindent \bf Bandit with Classifier ($\ucbc$):}
Figure~\ref{fig:BWC}(a) shows the average cumulative regret over time of three
algorithms.  Our baseline comparison is $\ucb$ which assumes that the
best search result is fixed throughout.  In addition, we compare to an algorithm we call
$\ucbo$, which uses the event oracle to reset $\ucb$ whenever an event occurs.  We also compared
to $\expthree$, but its performance was dramatically worse and thus we have
not included it in the figure.

\begin{figure}[t]
\begin{minipage}{0.48\linewidth}
\centering
\includegraphics[width=\linewidth]{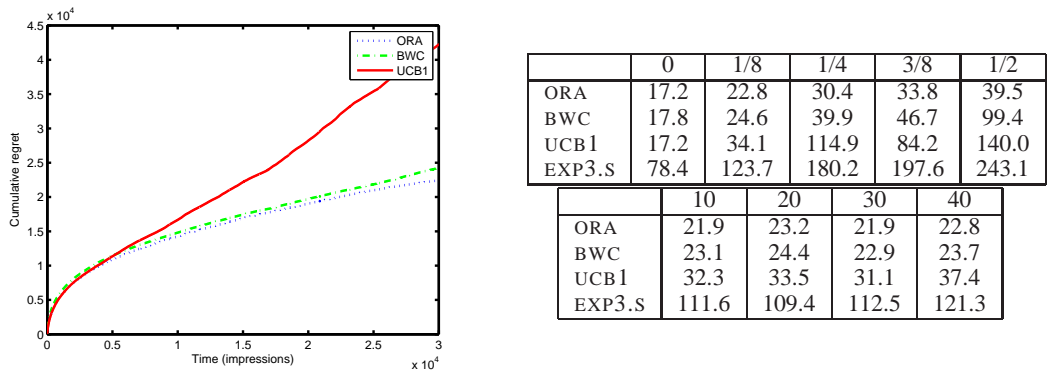}
\end{minipage}
\hfill
\begin{minipage}{0.48\linewidth}
\centering\small
\begin{tabular}{|l|c|c|c|c|c|} \hline
& 0  & 1/8  & 1/4 & 3/8 & 1/2 \\ \hline
$\ucbo$ & 17.2 & 22.8 & 30.4 & 33.8 & 39.5  \\
$\ucbc$& 17.8 & 24.6 & 39.9 & 46.7 & 99.4  \\
$\ucb$ & 17.2 & 34.1 & 114.9 & 84.2 & 140.0  \\
$\expthree$& 78.4 & 123.7 & 180.2 & 197.6 & 243.1  \\ \hline
\end{tabular}
\vfill
\begin{tabular}{|l|c|c|c|c|c|c|c|} \hline
&10&20&30&40 \\ \hline
$\ucbo$ & 21.9&23.2&21.9&22.8 \\
$\ucbc$ & 23.1&24.4&22.9&23.7 \\
$\ucb$ & 32.3&33.5&31.1&37.4 \\
$\expthree$ & 111.6&109.4&112.5&121.3 \\ \hline
\end{tabular}
\end{minipage}
\caption{
(a) (Left) $\ucbc$'s cumulative regret compared to $\ucb$ and $\ucbo$ ($\ucb$ with an oracle
indicating the exact locations of the intent-shifting event)
(b) (Right, Top Table) Final regret (in thousands) as the
fraction of intent-shifting queries varies.
With more intent-shifting queries, $\ucbc$'s advantage over prior approaches
improves.
(c) (Right, Bottom Table) Final regret (in thousands) as the number of features grows.
}
\label{fig:BWC}
\end{figure}

In the early stages of the experiment before any intent-shifting event has
happened, $\ucb$ performs the best.  $\ucbc$'s safe classifier makes many mistakes in the beginning and consequently pays the
price of believing that each query is experiencing an event when
in fact it is not.  As time progresses, $\ucbc$'s classifier makes fewer mistakes, and consequently knows when to reset $\ucb$ more accurately.  $\ucb$ alone ignores the context entirely and thus incurs substantially larger cumulative regret by the end.

{\noindent \bf Fraction of Intent-Shifting Queries:}
In the next experiment, we varied the fraction of intent-shifting queries.
Figure~\ref{fig:BWC}(b) shows the result of changing the
distribution from 0, 1/8, 1/4, 3/8 and 1/2
intent-shifting queries.  If there are no intent-shifting queries, then $\ucb$'s regret is the
best.  We expect this outcome since $\ucbc$'s classifier, because it is safe, initially assumes that
all queries are intent-shifting and thus needs time to learn that in fact no
queries are intent-shifting.  On the other hand, $\ucbc$'s regret dominates the
other approaches, especially as the fraction of intent-shifting queries
grows.  $\expthree$'s performance is quite poor in this experiment --- even
when all queries are intent-shifting.  The reason is that even when a query
is intent-shifting, there are at most 10 intent-shifting events, i.e., each query's
intent is not shifting all the time.

With more intent-shifting queries, the expectation is that regret
monotonically increases.  In general, this seems to be true in our
experiment.  There is however a decrease in regret going from 1/4 to
3/8 intent-shifting queries.  We believe that this is due to the fact that each
query has at most 10 intent-shifting events spread uniformly and it is
possible that there were fewer events with potentially smaller shifts
in intent in those runs.  In other words, the standard deviation of
the regret is large.  Over the ten 3/8 intent-shifting runs for
$\ucbo$, $\ucbc$, $\ucb$ and $\expthree$, the standard deviation was
roughly 1K, 10K, 12K and 6K respectively.

{\noindent \bf Number of Features:}
Finally, we comment on the performance of our approach as the number
of features grows.  Our theoretical results suggest that $\ucbc$'s performance
should deteriorate as the number of features grows.  Surprisingly,
$\ucbc$'s performance is consistently close to the Oracle's.
In Figure~\ref{fig:BWC}(b), we show the cumulative regret
after 3M impressions as the dimensionality of the context vector grows
from 10 to 40 features.  $\ucbc$'s regret
is consistently close to $\ucbo$ as the number of features grows.  On the
other hand, $\ucb$'s regret though competitive is worse than $\ucbc$, while $\expthree$'s performance is across
the board poor.  Note that both $\ucb$ and $\expthree$'s regret is completely
independent of the number of features.
The standard deviation of the regret over the 10 runs is substantially
lower than the previous experiment.  For example, over 10 features,
the standard deviation was 355, 1K, 5K, 4K for $\ucbo$, $\ucbc$, $\ucb$
and $\expthree$, respectively.

\section{Future Work}

The most immediate open question is whether we could train the classifier faster. One idea is to use a more efficient classifier, especially if we can relax the ``safety" requirement and somehow recover from false negatives. Another idea is to generate labeled samples not only upon positive predictions but upon negative ones as well, trading off the regret from additional exploration against the benefits of generating extra labeled samples. Finally, it would be desirable to supplement the existing worst-case provable guarantees with stronger ones for settings in which the contexts are sampled from a ``benign" distribution.

Theoretically, the main drawback of our approach is that we assume the existence of a ``perfect oracle" --- a deterministic boolean function on contexts which correctly predicts whether a temporal event has occurred in the current round. It is desirable to extend our results to scenarios in which the contexts allow only approximate or probabilistic prediction. Even though such contexts contain useful signal, exploiting this signal for our purposes appears quite challenging. In particular, it seems to require making the ``bandit plus classifier" setup resilient against (infrequent) incorrectly labeled samples and perhaps also against (infrequent) false negatives. It should be noted that the aforementioned resiliency can potentially lead to large improvements in the present oracle-based setting as well, as we might be able to deploy much more efficient classifiers.

Empirically, the main question left for future work is testing the ``bandit plus classifier" approach in a realistic setting. The challenge here is two-fold. First, one needs to select which features to use for contexts, and verify experimentally how informative they are in predicting the temporal events. Second, since gaining access to live search traffic is difficult, one would need to simulate it using the search logs, the difficulty being is that the search logs might not have enough data points for alternatives that have not been chosen frequently by the search engine.



{\bf Acknowledgements.}
We thank Rakesh Agrawal, Alan Halverson, Krishnaram Kenthapadi, Robert Kleinberg, Robert Schapire and Yogi Sharma for their helpful comments and suggestions.

\begin{small}
{\def\section*#1{}
\subsubsection*{References}

}
\end{small}

\appendix

\section{Details for the proof of Theorem~\ref{thm:LB-eventful}(ii)}
\label{app:LB}

\begin{claim} \label{cl:contexts}
We can define context sequences $\{x^0_t\}$ and $\{x^1_t\}, \ldots, \{x^N_t\}$ with the following properties: (1) each sequence $\{x^i_t\}$, when paired with a problem instance $\mathcal{I}_i$, defines an eventful bandit problem consistent with all our assumptions, and (2) the sequences $\{x^0_t\}$ and $\{x^i_i\}$ agree through the first $i$ phases.
\end{claim}
\begin{proof} Let $y_1, \ldots, y_{d_\CF} \in \CX$ be a sequence of contexts such that $y_j \notin S_\CF(\{y_1, \ldots, y_{j-1}\})$ for all $j = 1, \ldots, d_\CF$. We know this sequence exists by the definition of $d_\CF$. Also assume there exists an ``always negative'' context $x^-$ such that $f(x^-) = -1$ for all $f \in \CF$ (this assumption is not necessary, but is convenient). Let $t_j$ be the first round of phase $j$.

Define $\{x^0_t\}$ as follows: let $x^0_{t_j} = y_j$ for each phase $1 \le j \le N$, and let $x^0_t = x^-$ for all other rounds.

For $1 \le i \le N$, define $\{x^i_t\}$ as follows: let $x^i_{t_j} = y_j$ for each phase $1 \le j \le i$, and let $x^i_t = x^-$ for all other rounds.\end{proof}

\begin{claim} \label{cl:event}
$(1-q_{i,i}) \ge T^{-1/3} (1-q_{i,0})$
\end{claim}

\begin{proof} Throughout this proof, we fix phase $i$. Define a \emph{realization} to be a particular sequence of outcomes of all random samples from click distributions, as well as all random choices (if any), during an execution of algorithm $A$ \emph{through the end of phase $i$}. For example, if $s = s_1, \ldots, s_M$ is a realization, then $s_1$ might correspond to the click observed in the first round, $s_2, \ldots, s_5$ might correspond to random choices made by the algorithm, $s_6$ might correspond to the click observed in the second round, and so on. By the chain rule, for any $j \in \{0, i\}$:

$$\textstyle \Pr_{\mathcal{I}_j}[s] = \Pr_{\mathcal{I}_j}[s_1] \Pr_{\mathcal{I}_j}[s_2 | s_1]\cdots\Pr_{\mathcal{I}_j}[s_M | s_1,\ldots,s_{M-1}]$$

For any realization $s$, let $\Pr_{\mathcal{I}_j}[s_\alpha]$ be the product of terms in the above product that correspond to outcomes other than observed clicks in phase $i$. Let $\CS$ be the set of realizations in which arm $z$ is selected by $\mathcal{A}$ less than $\log T$ times in phase $i$. Let $n_{a,c}(s)$ be the number of times in realization $s$ that arm $a \in \{y, z\}$ is selected in phase $i$ and payoff $c \in \{0, 1\}$ is observed as a result. Then

\begin{align*}
\textstyle (1-q_{i,0}) & = \sum_{s \in \CS} \textstyle \Pr_{\mathcal{I}_0}[s] \\
& = \sum_{s \in \CS} \textstyle \Pr_{\mathcal{I}_0}[s_\alpha] \left(\frac{1}{2}\right)^{n_{y, 0}(s)}\left(\frac{1}{2}\right)^{n_{y, 1}(s)}\left(\frac{\ce}{1+\ce}\right)^{n_{z, 0}(s)}\left(\frac{1}{1+\ce}\right)^{n_{z, 1}(s)}\\
& = \sum_{s \in \CS} \textstyle \Pr_{\mathcal{I}_0}[s_\alpha] \left(\frac{1}{2}\right)^{n_{y, 0}(s)}\left(\frac{1}{2}\right)^{n_{y, 1}(s)}\left(\ce\cdot\frac{1}{1+\ce}\right)^{n_{z, 0}(s)}\left(\frac{1}{\ce}\cdot\frac{\ce}{1+\ce}\right)^{n_{z, 1}(s)}\\
& \le \sum_{s \in \CS} \textstyle \Pr_{\mathcal{I}_0}[s_\alpha] (\ce)^{n_{z,0}(s)} \left(\frac{1}{2}\right)^{n_{y, 0}(s)}\left(\frac{1}{2}\right)^{n_{y, 1}(s)}\left(\frac{1}{1+\ce}\right)^{n_{z, 0}(s)}\left(\frac{\ce}{1+\ce}\right)^{n_{z, 1}(s)}\\
& \le (\ce)^{\log T} \sum_{s \in \CS} \textstyle \Pr_{\mathcal{I}_0}[s_\alpha] \left(\frac{1}{2}\right)^{n_{y, 0}(s)}\left(\frac{1}{2}\right)^{n_{y, 1}(s)}\left(\frac{1}{1+\ce}\right)^{n_{z, 0}(s)}\left(\frac{\ce}{1+\ce}\right)^{n_{z, 1}(s)}\\
& = T^{1/3} \sum_{s \in \CS} \textstyle \Pr_{\mathcal{I}_i}[s] \\
& = T^{1/3} \textstyle (1-q_{i,i}) \qedhere
\end{align*}
\end{proof}

\end{document}